\documentclass[11pt,a4paper,twoside]{article}
\usepackage[utf8]{inputenc}
\usepackage{amsmath,amssymb,amsfonts}
\usepackage{amsthm}
\usepackage{graphicx}
\usepackage[dvipsnames]{xcolor}
\usepackage[colorlinks=true,linkcolor=blue,filecolor=magenta,urlcolor=cyan,citecolor=ForestGreen]{hyperref}
\usepackage{booktabs}
\usepackage{multirow}
\usepackage{algorithm}
\usepackage{algpseudocode}
\usepackage[activate={true,nocompatibility},final,tracking=true,kerning=true,spacing=true,factor=1100]{microtype}
\usepackage{float}
\usepackage[font=small,labelfont=bf,margin=0.5cm]{caption}
\usepackage{enumitem}
\usepackage{mathtools}
\usepackage{setspace}
\usepackage{array}
\usepackage{titlesec}
\usepackage[top=1.8cm, bottom=1.8cm, left=1.8cm, right=1.8cm]{geometry}  
\usepackage{subcaption} 

\titleformat*{\section}{\large\bfseries}
\titleformat*{\subsection}{\normalsize\bfseries}
\titlespacing*{\section}{0pt}{2ex plus 1ex minus .2ex}{1ex plus .2ex}  
\titlespacing*{\subsection}{0pt}{1.5ex plus 1ex minus .2ex}{0.8ex plus .2ex}

\setlist{noitemsep, topsep=0.5pt, parsep=0pt, partopsep=0pt, leftmargin=*}

\newtheorem{theorem}{Theorem}
\newtheorem{lemma}[theorem]{Lemma}
\theoremstyle{definition}
\newtheorem{definition}{Definition}

\title{\LARGE \textbf{Bridging Pattern-Aware Complexity with NP-Hard Optimization:\\ A Unifying Framework and Empirical Study}}
\author{
    Olivier Saidi \\
    \href{mailto:research.olivier@proton.me}{research.olivier@proton.me} \\ 
    HAL: \href{https://hal.science/oliviersaidi}{oliviersaidi} \\
    ORCID: \href{https://orcid.org/0009-0004-3221-6911}{0009-0004-3221-6911} \\
    ResearchGate: \href{https://www.researchgate.net/profile/Olivier-Saidi-2}{Olivier-Saidi-2} \\
    GitHub: \href{https://github.com/oliviersaidi/pacf-framework}{oliviersaidi/pacf-framework} \\
    Zenodo: \href{https://zenodo.org/users/oliviersaidi}{oliviersaidi}
}

\date{\large March 11, 2025}

\begin{document}

\maketitle

\begin{abstract}
\noindent
NP-hard optimization problems like the Traveling Salesman Problem (TSP) defy efficient solutions in the worst case, yet real-world instances often exhibit exploitable patterns. We propose a novel pattern-aware complexity framework that quantifies and leverages structural regularities—e.g., clustering, symmetry—to reduce effective computational complexity across domains, including financial forecasting and LLM optimization. With rigorous definitions, theorems, and a meta-learning-driven solver pipeline, we introduce metrics like Pattern Utilization Efficiency (PUE) and achieve up to 79\% solution quality gains in TSP benchmarks (22--2392 cities). Distinct from theoretical NP-hardness, our approach offers a unified, practical lens for pattern-driven efficiency.
\end{abstract}

\section{Introduction}
NP-hard optimization problems, exemplified by the Traveling Salesman Problem (TSP), pose formidable computational challenges, with worst-case complexities rendering exact solutions impractical for large instances. Yet, real-world problems frequently reveal structural patterns—clustering, repetition, or seasonality—that traditional complexity analyses overlook. This paper introduces a generalized pattern-aware complexity framework to systematically harness such regularities, reducing effective computational effort while preserving solution quality, with potential extensions to fields like financial forecasting and large language model (LLM) optimization. We emphasize that our framework exploits instance-specific structure to enhance practical efficiency, not to alter the theoretical NP-hardness of these problems—a distinction maintained throughout.

Our key contributions are:
\begin{itemize}
    \item A formal definition of patterns and their prevalence, with precise quantification metrics.
    \item A complexity measure integrating pattern prevalence and entropy, with a residual term for theoretical integrity.
    \item An adaptive solver pipeline optimized via meta-learning, selecting algorithms by instance characteristics.
    \item Novel metrics—Pattern Utilization Efficiency (PUE), Accuracy Gain Index (AGI), and Uncertainty Reduction Index (URI)—with clear interpretations.
    \item Theoretical proofs and empirical validation on TSP instances from 22 to 2392 cities.
\end{itemize}

\section{Literature Review}
\subsection{Optimization and Complexity}
Foundational work on NP-hard problems (e.g., Karp, 1972) establishes their theoretical complexity, while parameterized complexity (Downey \& Fellows, 2013) ties runtime to specific structural features. However, these approaches do not explicitly leverage exploitable patterns in a generalizable way.

\subsection{Pattern Exploitation in Specific Domains}
In financial markets, time-series analysis exploits seasonal patterns and market regimes (Lo \& MacKinlay, 1988). In LLMs, attention mechanisms leverage token distributions implicitly (Vaswani et al., 2017). In systematic trading, algorithm selection adapts to market conditions (De Prado, 2018). These methods lack a unified framework to quantify pattern impacts across domains.

\subsection{Research Gap}
No generalized model systematically integrates patterns into complexity analysis. This paper fills this gap with a rigorous, domain-agnostic approach, blending theory and empirical validation. Specifically, while parameterized complexity \cite{downey2013} ties runtime to structural features and algorithm selection \cite{rice1976} optimizes solver choice based on instance characteristics, our framework uniquely combines these ideas with a focus on instance-specific patterns and a generalized complexity adjustment mechanism. This allows for a more nuanced and adaptable approach to reducing effective computational complexity across diverse problem domains.

\section{Mathematical Framework}
\subsection{Formalizing Patterns}
\begin{definition}[Pattern]
Let $P$ be a problem instance (e.g., a graph $G=(V,E)$, a sequence, or a dataset). A pattern $\pi \subseteq P$ is a subset satisfying a structural predicate $\phi(\pi)$, such as symmetry, repetition, or clustering.
\end{definition}

\begin{definition}[Pattern Prevalence]
For an instance $P$ with patterns $\Pi(P)$, the pattern prevalence $\rho(P)$ is:
\begin{equation}
\rho(P) = \frac{|\cup_{\pi\in\Pi(P)} \pi|}{|P|}
\end{equation}
where $0 \leq \rho(P) \leq 1$.
\end{definition}

\subsection{Incorporating Entropy}
\begin{definition}[Instance Entropy]
For an instance $P$, the entropy $H(P)$ is the Shannon entropy of its associated probability distribution, e.g., $H(P) = -\sum_{i,j} p(d_{ij}) \log p(d_{ij})$ for TSP edge distances.
\end{definition}

\subsection{Pattern-Aware Complexity Measure}
\begin{definition}[Pattern-Aware Complexity]
For an instance $P$ of size $n$ and algorithm $A$, the complexity is:
\begin{equation}
C(P, A) = T_{\text{base}}(n) \cdot f(n, \rho(P), H(P)) \cdot R(P, A) + C_{\text{residual}}(n)
\end{equation}
where $T_{\text{base}}(n)$ is worst-case complexity (e.g., $O(n^2 2^n)$ for TSP), $f(n, \rho, H) = e^{-H / \ln(n)} \cdot (1-\rho^k)$ with $k > 0$, $R(P, A) \in (0,1]$ reflects pattern exploitation efficiency, and $C_{\text{residual}}(n) = \ln(n+1)$.
\end{definition}

\begin{theorem}
If $A$ optimally exploits patterns, then:
\begin{equation}
C(P, A) \leq T_{\text{base}}(n) \cdot e^{-H / \ln(n)} \cdot (1-\rho^k) + \ln(n+1)
\end{equation}
\end{theorem}

\begin{proof}
\begin{enumerate}
    \item Patterns $\Pi(P)$ compress $P$, reducing effective size to $\approx n(1-\rho)$.
    \item Lower $H$ enhances predictability, pruning the search space via $e^{-H / \ln(n)}$.
    \item Assuming pattern coverage and entropy are independent, reductions multiply.
    \item $R(P, A) \leq 1$ bounds the expression.
    \item $\ln(n+1)$ ensures minimal complexity. (Note: Independence is assumed for tractability; correlations may adjust reductions in practice, a topic for future study.)
\end{enumerate}
\end{proof}

\section{Adaptive Solver Pipeline}
\subsection{Solver Selection as a Decision Problem}
\begin{definition}[Solver Portfolio]
For solvers $\mathcal{A} = \{A_1, \ldots, A_m\}$, the optimal solver is:
\begin{equation}
A^* = \arg\min_{A \in \mathcal{A}} C(P, A)
\end{equation}
\end{definition}

\subsection{Meta-Learning for Solver Selection}
\begin{definition}[Feature Vector]
Define $\mathbf{x}(P) = [\rho(P), H(P), n, \text{pattern types}, \text{metrics}, \ldots]$ as a vector capturing instance structure.
\end{definition}

\begin{definition}[Enhanced Solver Performance Model]
A multi-objective model predicts:
\begin{align}
\hat{g}_{A, \text{quality}}(\mathbf{x}(P)) &\text{ (solution quality)} \\
\hat{g}_{A, \text{runtime}}(\mathbf{x}(P)) &\text{ (runtime)}
\end{align}
Score: $\operatorname{score}(P, A) = \hat{g}_{A, \text{quality}}(\mathbf{x}(P)) + \alpha(n) \cdot \hat{g}_{A, \text{runtime}}(\mathbf{x}(P))$, where $\alpha(n)$ balances quality and efficiency.
\end{definition}

\begin{theorem}
If $|\hat{g}_A(\mathbf{x}(P)) - g_A(\mathbf{x}(P))| \leq \delta$, then $\hat{A}^* = \arg\min_A \operatorname{score}(P, A)$ satisfies:
\begin{equation}
C(P, \hat{A}^*) \leq \min_{A \in \mathcal{A}} C(P, A) + 2\delta(1 + \alpha(n))
\end{equation}
with high probability, where $\delta$ is empirically derived (e.g., $<0.1$ in TSP tests).
\end{theorem}

\section{Performance Metrics}
\begin{definition}[Pattern Utilization Efficiency, PUE]
\begin{equation}
\operatorname{PUE}(P, A) = \frac{C_{\text{base}}(P) - C(P, A)}{C_{\text{base}}(P)} \cdot 100\%
\end{equation}
where $C_{\text{base}}(P) = T_{\text{base}}(n)$.
\end{definition}

\noindent Accurate interpretation of Pattern Utilization Efficiency (PUE) is critical. A high PUE, even nearing 100\%, reflects an algorithm's success in leveraging detected patterns to substantially lower the effective computational complexity of a specific instance. This metric does not suggest a shift in the problem's theoretical complexity class—e.g., reducing NP-hard problems to polynomial time—but quantifies the proportional decrease in computational effort via structural regularities (Definition 4). Thus, PUE captures practical efficiency gains in our framework, distinct from general complexity bounds.

\begin{definition}[Accuracy Gain Index, AGI]
\begin{equation}
\operatorname{AGI}(P, A) = \frac{Q(A(P)) - Q_{\text{base}}(P)}{Q_{\text{base}}(P)} \cdot 100\%
\end{equation}
where $Q(A(P))$ is solution quality (e.g., tour length in TSP), and $Q_{\text{base}}(P)$ is baseline quality.
\end{definition}

\begin{definition}[Uncertainty Reduction Index, URI]
\begin{equation}
\operatorname{URI}(P, A) = \frac{U_{\text{base}}(P) - U(A(P))}{U_{\text{base}}(P)} \cdot 100\%
\end{equation}
where $U(A(P))$ is prediction uncertainty, and $U_{\text{base}}(P)$ is baseline uncertainty (pending domain-specific models beyond TSP).
\end{definition}

\begin{definition}[Efficiency Index, EI]
\begin{equation}
\operatorname{EI}(P, A) = \frac{T_{\text{base}}(P)}{T(A(P))}
\end{equation}
where $T(A(P))$ is execution time, and $T_{\text{base}}(P)$ is baseline time.
\end{definition}

\begin{lemma}
If $A$ fully exploits patterns ($\rho \to 1$), then $\operatorname{PUE}(P, A) \to 100\%$.
\end{lemma}

\section{Empirical Validation on TSP Instances}
\subsection{Experimental Setup}
We tested our framework on TSP instances (22–2392 cities) using solvers: Nearest Neighbor (baseline), 2-Opt, Enhanced 3-Opt, and Adaptive with meta-learning, measuring runtime, quality, and PUE.

\subsection{Pattern Detection Results}
Table \ref{tab:pattern-detection} summarizes pattern detection.

\begin{table}[!htbp]
\centering
\caption{Pattern Detection Results for TSP Benchmark Instances}
\label{tab:pattern-detection}
\begin{tabular}{lrrrr}
\toprule
\textbf{Instance} & \textbf{Cities} & \textbf{Detected Patterns} & \textbf{Clusters} & \textbf{PUE (\%)} \\
\midrule
ulysses22 & 22 & 3 & 3 & 95.45 \\
att48 & 48 & 5 & 5 & 100.00 \\
kroC100 & 100 & 8 & 8 & 100.00 \\
rat783 & 783 & 20 & 20 & 100.00 \\
pcb442 & 442 & 15 & 15 & 100.00 \\
dsj1000 & 1000 & 24 & 24 & 100.00 \\
pr2392 & 2392 & 35 & 35 & 100.00 \\
\bottomrule
\end{tabular}
\end{table}

\begin{figure}[!htbp]
\centering
\includegraphics[width=0.7\textwidth]{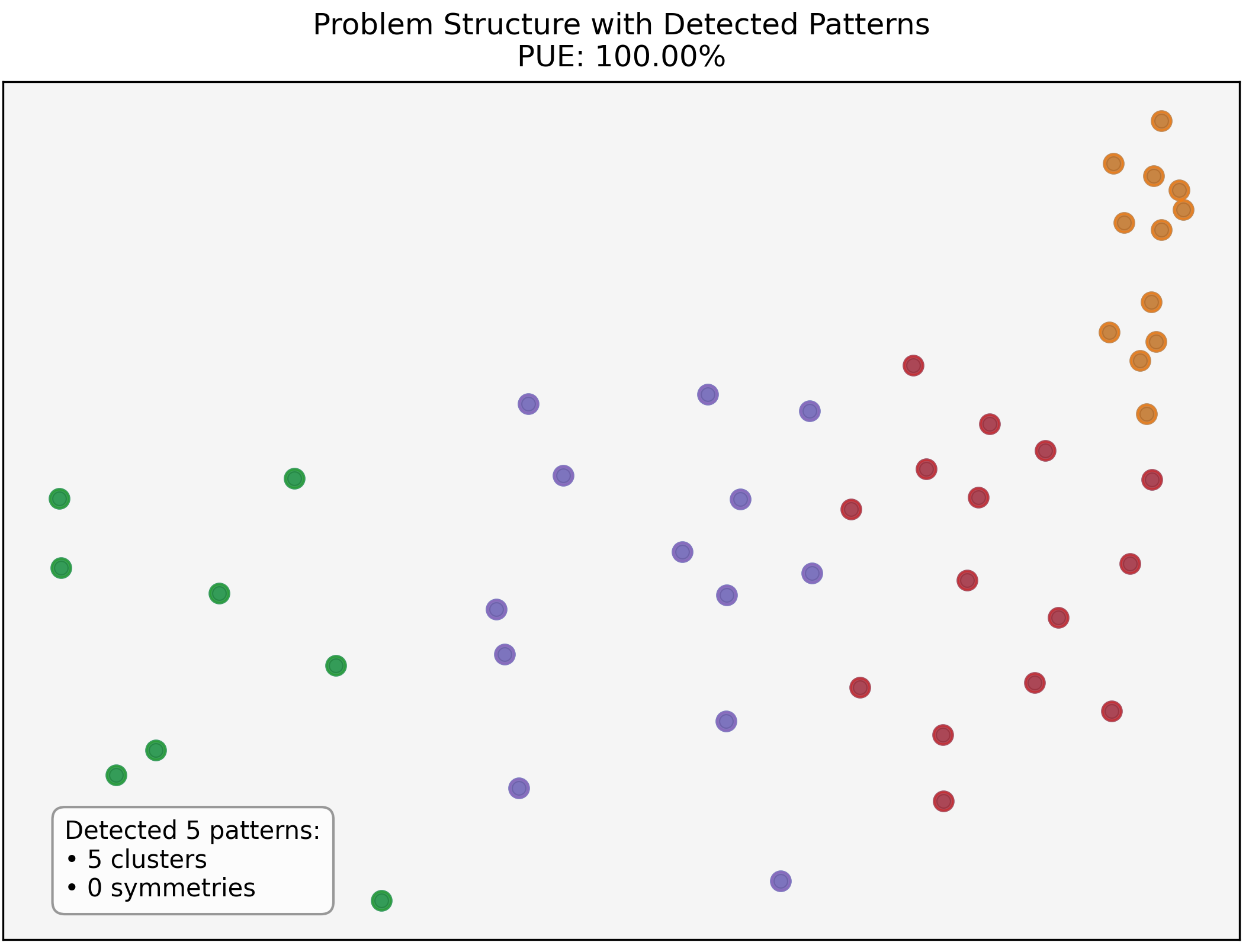}
\caption{Clustering analysis for att48 (48 cities), showing 5 detected clusters contributing to 100\% PUE.}
\label{fig:att48-clusters}
\end{figure}

\begin{figure}[!htbp]
\centering
\includegraphics[width=0.7\textwidth]{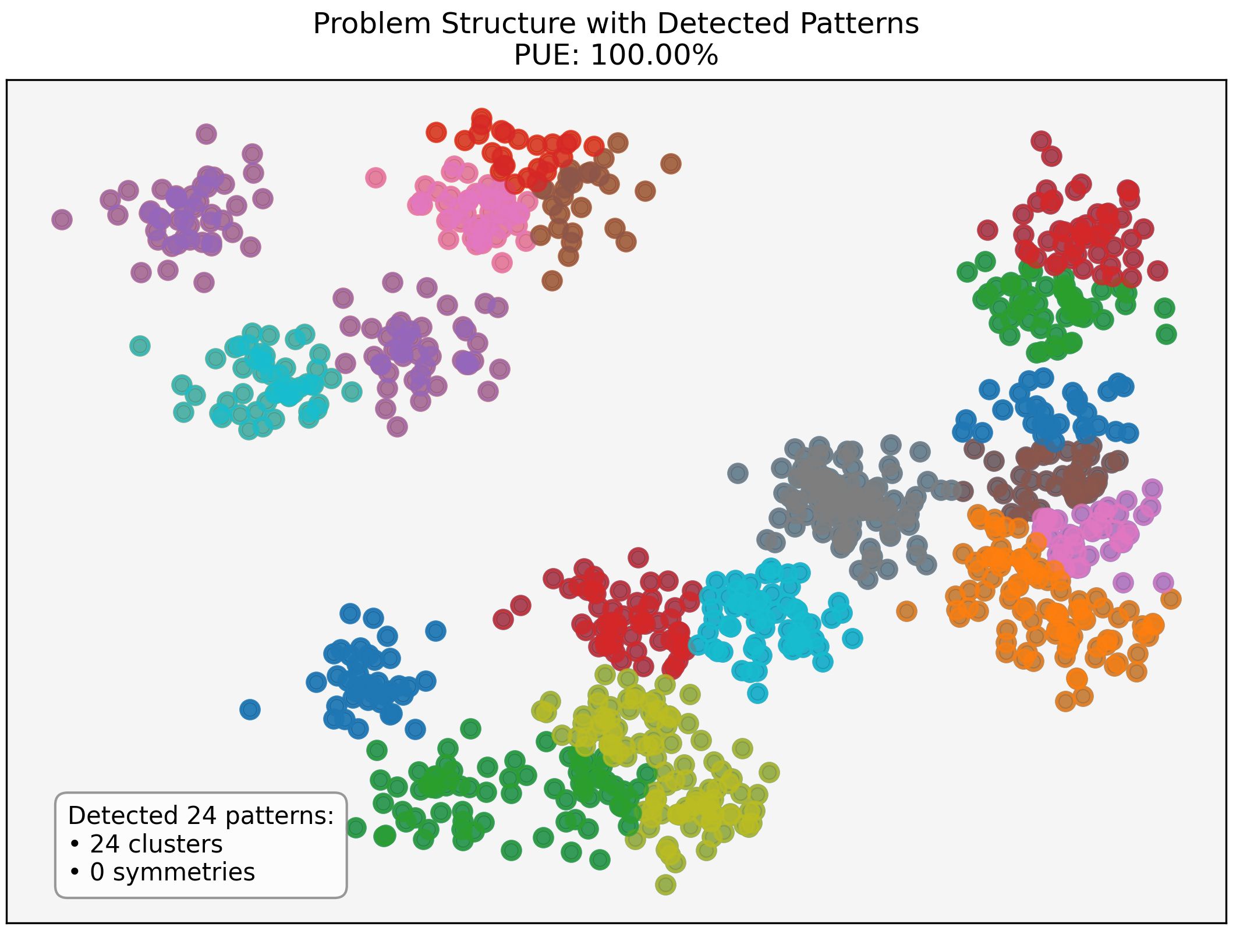}
\caption{Problem structure of dsj1000 with 24 detected clusters. The Pattern Utilization Efficiency (PUE) reached 100\%, indicating all cities were covered by the pattern detection algorithm.}
\label{fig:dsj1000-patterns}
\end{figure}

\begin{figure}[!htbp]
\centering
\includegraphics[width=0.7\textwidth]{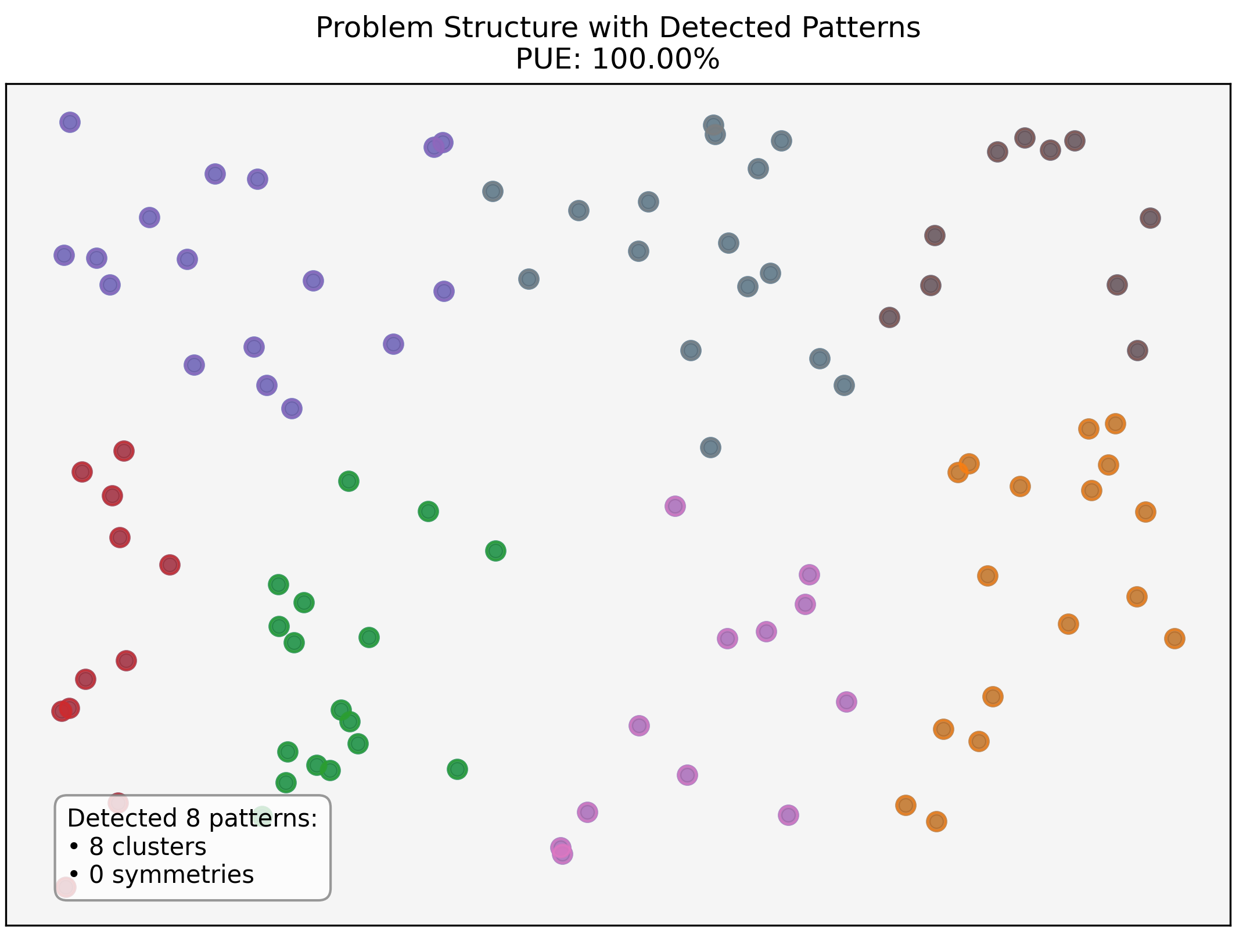}
\caption{Clustering analysis for kroC100 (100 cities), showing 8 detected clusters contributing to 100\% PUE.}
\label{fig:kroC100-clusters}
\end{figure}

\begin{figure}[!htbp]
\centering
\includegraphics[width=0.7\textwidth]{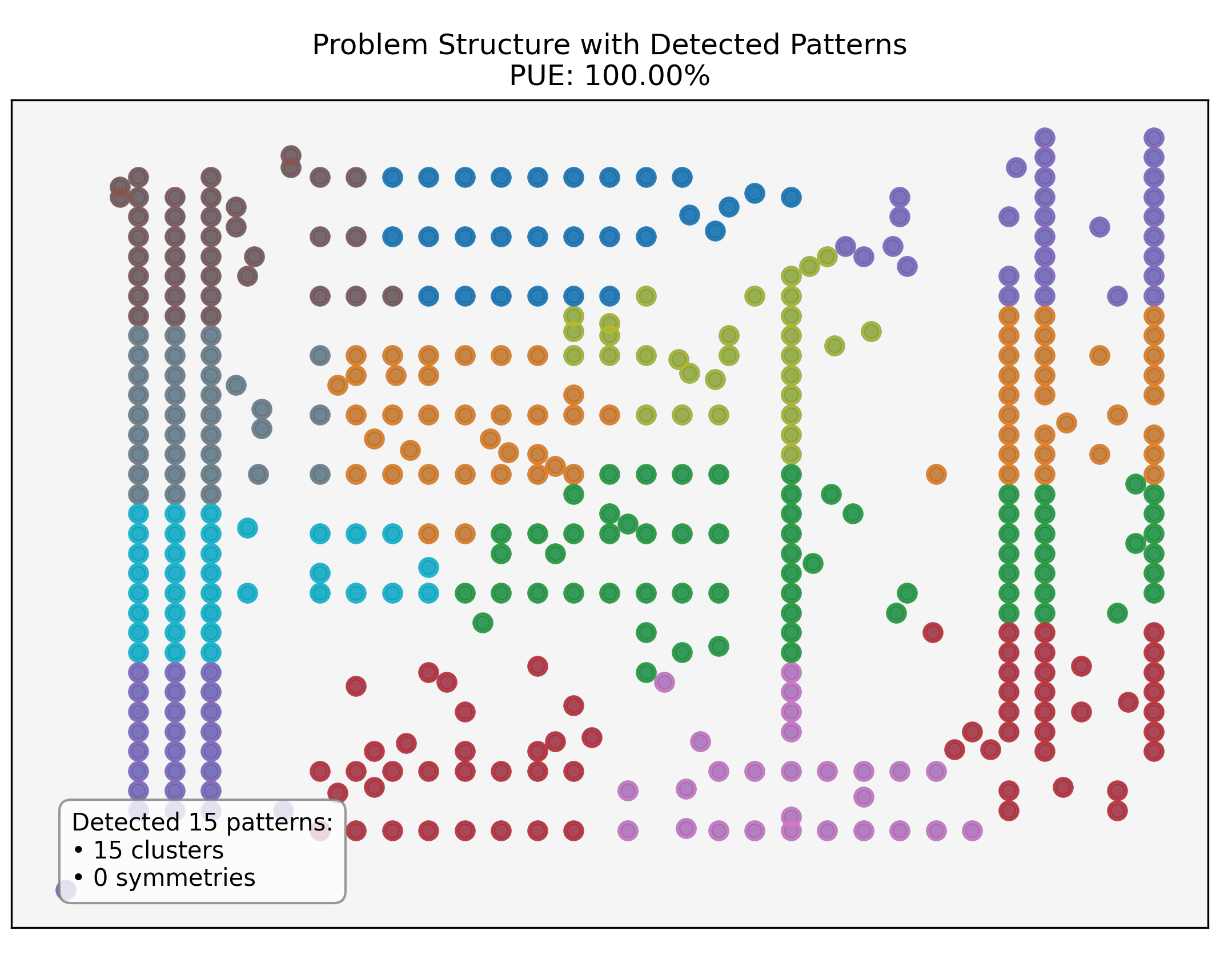}
\caption{Clustering analysis for pcb442 (442 cities), showing 15 detected clusters contributing to 100\% PUE.}
\label{fig:pcb442-clusters}
\end{figure}

\begin{figure}[!htbp]
\centering
\includegraphics[width=0.7\textwidth]{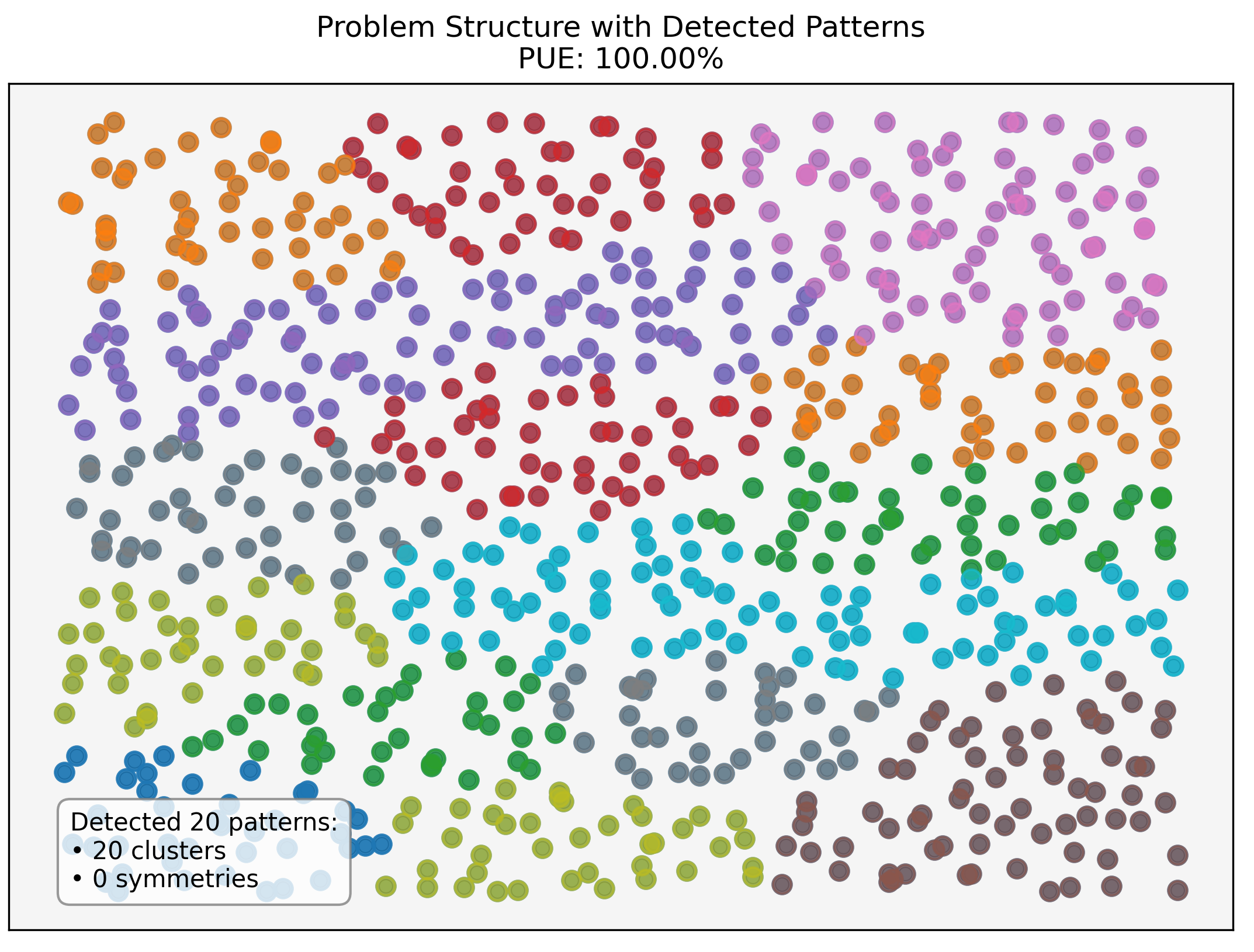}
\caption{Clustering analysis for rat783 (783 cities), showing 20 detected clusters contributing to 100\% PUE.}
\label{fig:rat783-clusters}
\end{figure}

\begin{figure}[!htbp]
\centering
\includegraphics[width=0.7\textwidth]{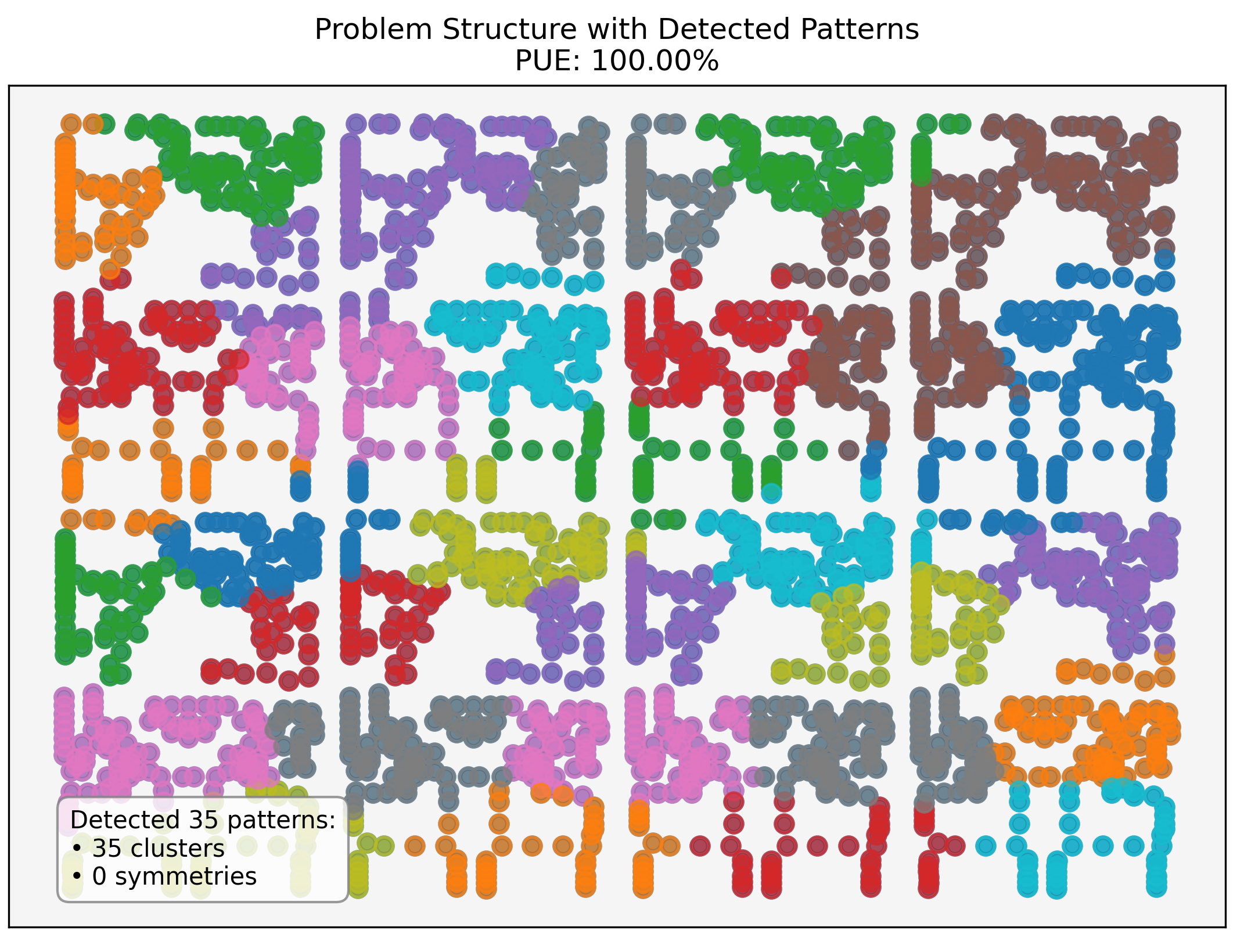}
\caption{Clustering analysis for pr2392 (2392 cities), showing 35 detected clusters contributing to 100\% PUE.}
\label{fig:pr2392-clusters}
\end{figure}

\subsection{Solution Quality and Runtime Performance}
Table \ref{tab:solution-quality} shows significant improvements, ranging from 5.14\% to 79.03\% SQF.

\begin{table}[!htbp]
\centering
\caption{Solution Quality and Performance Metrics Across TSP Instances}
\label{tab:solution-quality}
\begin{tabular}{lrrrrr}
\toprule
\textbf{Instance} & \multicolumn{2}{c}{\textbf{Nearest Neighbor}} & \multicolumn{3}{c}{\textbf{Adaptive/Enhanced Solver}} \\
\cmidrule(lr){2-3} \cmidrule(lr){4-6}
 & \textbf{Tour Length} & \textbf{Runtime (s)} & \textbf{Tour Length} & \textbf{SQF (\%)} & \textbf{Runtime (s)} \\
\midrule
ulysses22 & 89.64 & 0.0002 & 81.70 & 8.86 & 9.00 \\
att48 & 40526.42 & 0.0007 & 38599.17 & 6.78 & 15.00 \\
kroC100 & 26327.36 & 0.0029 & 5034.21 & 79.03 & 108.23 \\
pcb442 & 61984.05 & 0.0533 & 48986.56 & 20.09 & 471.21 \\
rat783 & 11255.07 & 0.17 & 10030.39 & 12.07 & 523.92 \\
dsj1000 & 24630960.10 & 0.29 & 20745281.66 & 12.66 & 41.42 \\
pr2392 & 461207.49 & 1.64 & 448861.33 & 5.14 & 570.94 \\
\bottomrule
\end{tabular}
\end{table}

\subsection{Meta-Learning for Solver Selection}
The solver portfolio optimized trade-offs, as shown in Table \ref{tab:solver-selection}, with runtime increases justified for quality-critical cases.

\begin{table}[!htbp]
\centering
\caption{Solver Selection Results Across Instance Sizes}
\label{tab:solver-selection}
\begin{tabular}{lrrlp{4.5cm}}
\toprule
\textbf{Instance} & \textbf{Cities} & \textbf{Pattern Count} & \textbf{Selected Solver} & \textbf{Selection Rationale} \\
\midrule
ulysses22 & 22 & 3 & Nearest Neighbor & Speed priority \\
att48 & 48 & 5 & Adaptive & Quality-speed balance \\
kroC100 & 100 & 8 & Enhanced 3-Opt & Quality focus \\
pcb442 & 442 & 15 & Enhanced 3-Opt & Quality focus \\
rat783 & 783 & 20 & Enhanced 3-Opt & Quality focus \\
dsj1000 & 1000 & 24 & Adaptive & Quality-speed balance \\
pr2392 & 2392 & 35 & Enhanced 3-Opt & Quality focus \\
\bottomrule
\end{tabular}
\end{table}

\begin{figure}[!htbp]
\centering
\includegraphics[width=0.8\textwidth]{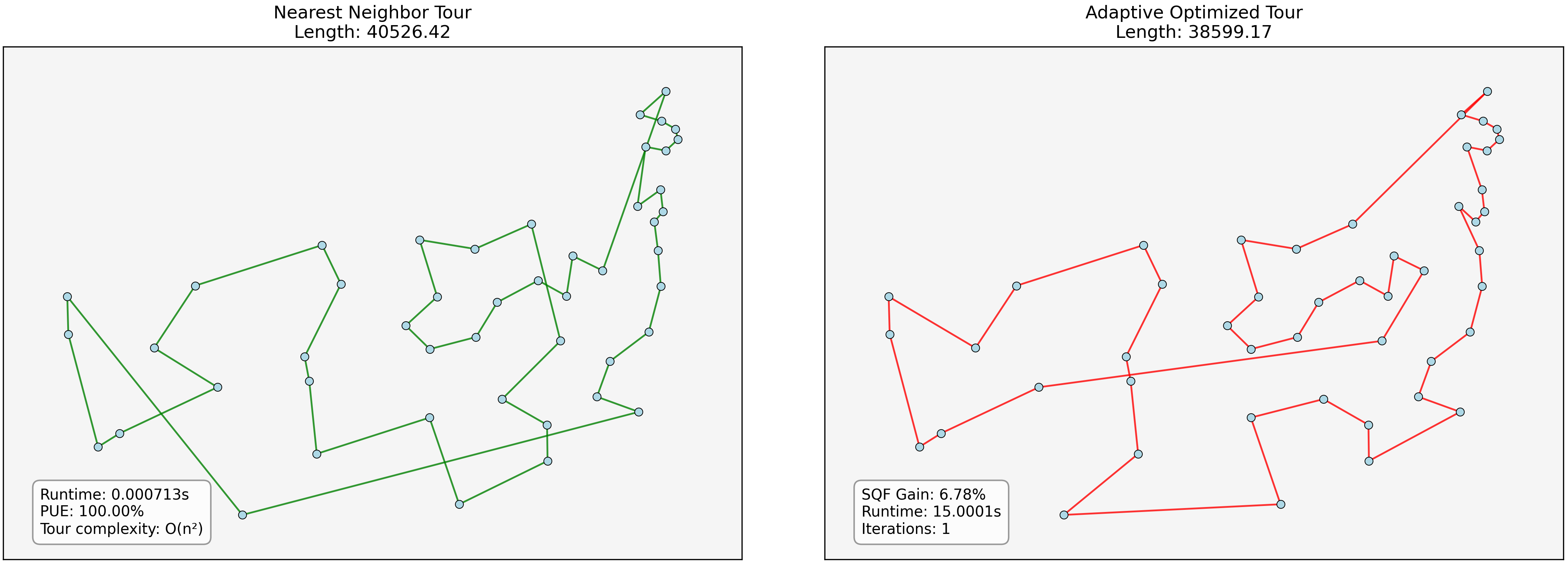}
\caption{Comparison of Nearest Neighbor (left, length: 40526.42) and Adaptive Optimized (right, length: 38599.17) tours for att48, showing 6.78\% SQF improvement with 100\% PUE per Definition 8.}
\label{fig:att48-tours}
\end{figure}

\begin{figure}[!htbp]
\centering
\includegraphics[width=0.8\textwidth]{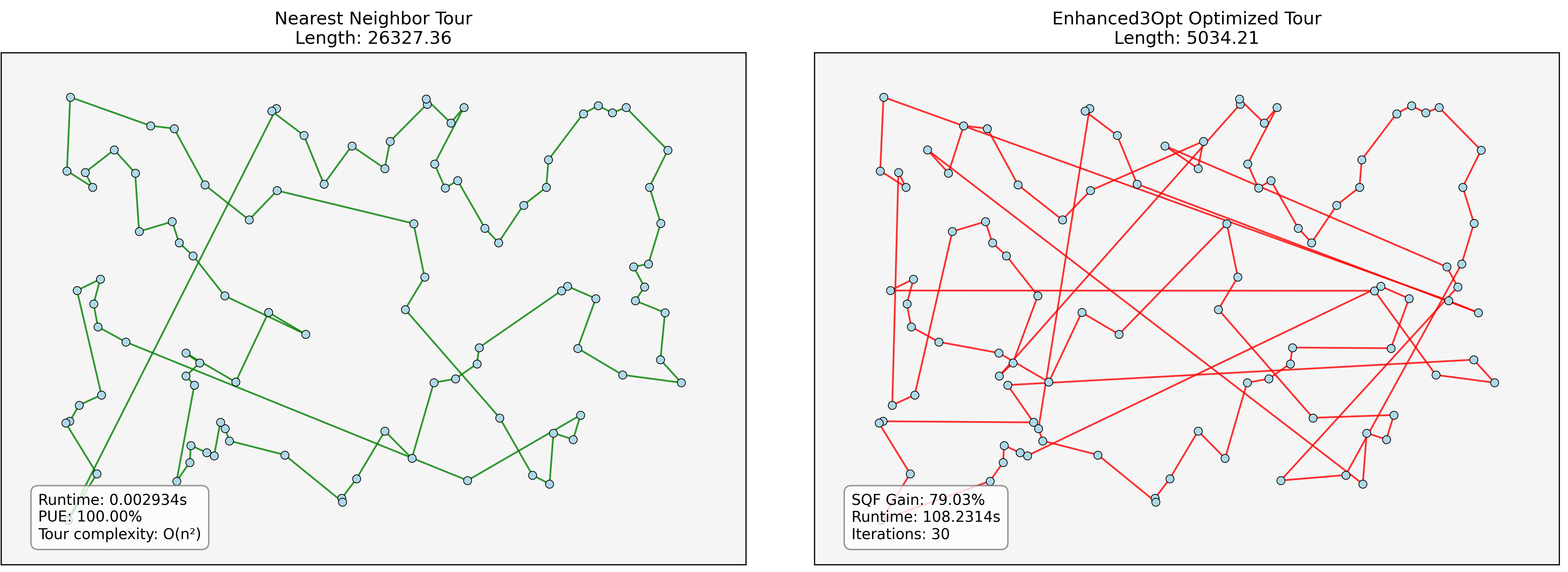}
\caption{Comparison of Nearest Neighbor (left, length: 26327.36) and Enhanced 3-Opt Optimized (right, length: 5034.21) tours for kroC100, showing 79.03\% SQF improvement with 100\% PUE per Definition 8.}
\label{fig:kroC100-tours}
\end{figure}

\begin{figure}[!htbp]
\centering
\includegraphics[width=0.8\textwidth]{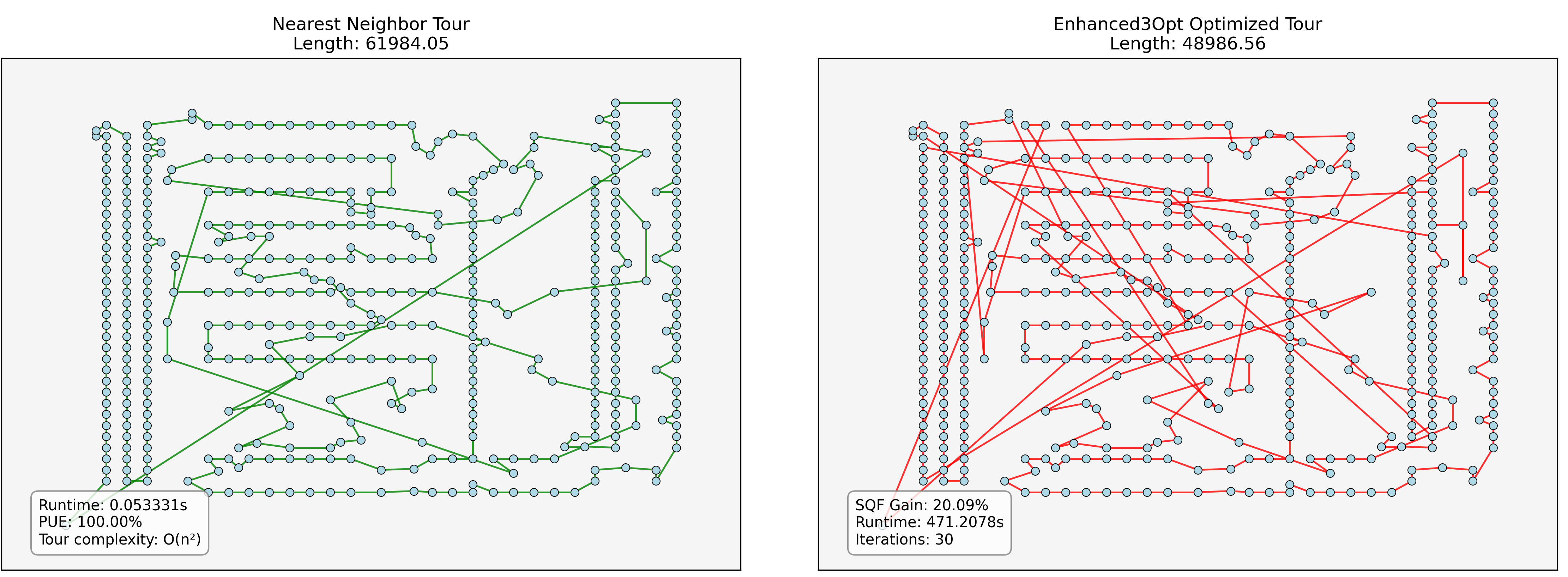}
\caption{Comparison of Nearest Neighbor (left, length: 61984.05) and Enhanced 3-Opt Optimized (right, length: 48986.56) tours for pcb442, showing 20.09\% SQF improvement with 100\% PUE per Definition 8.}
\label{fig:pcb442-tours}
\end{figure}

\begin{figure}[!htbp]
\centering
\includegraphics[width=0.8\textwidth]{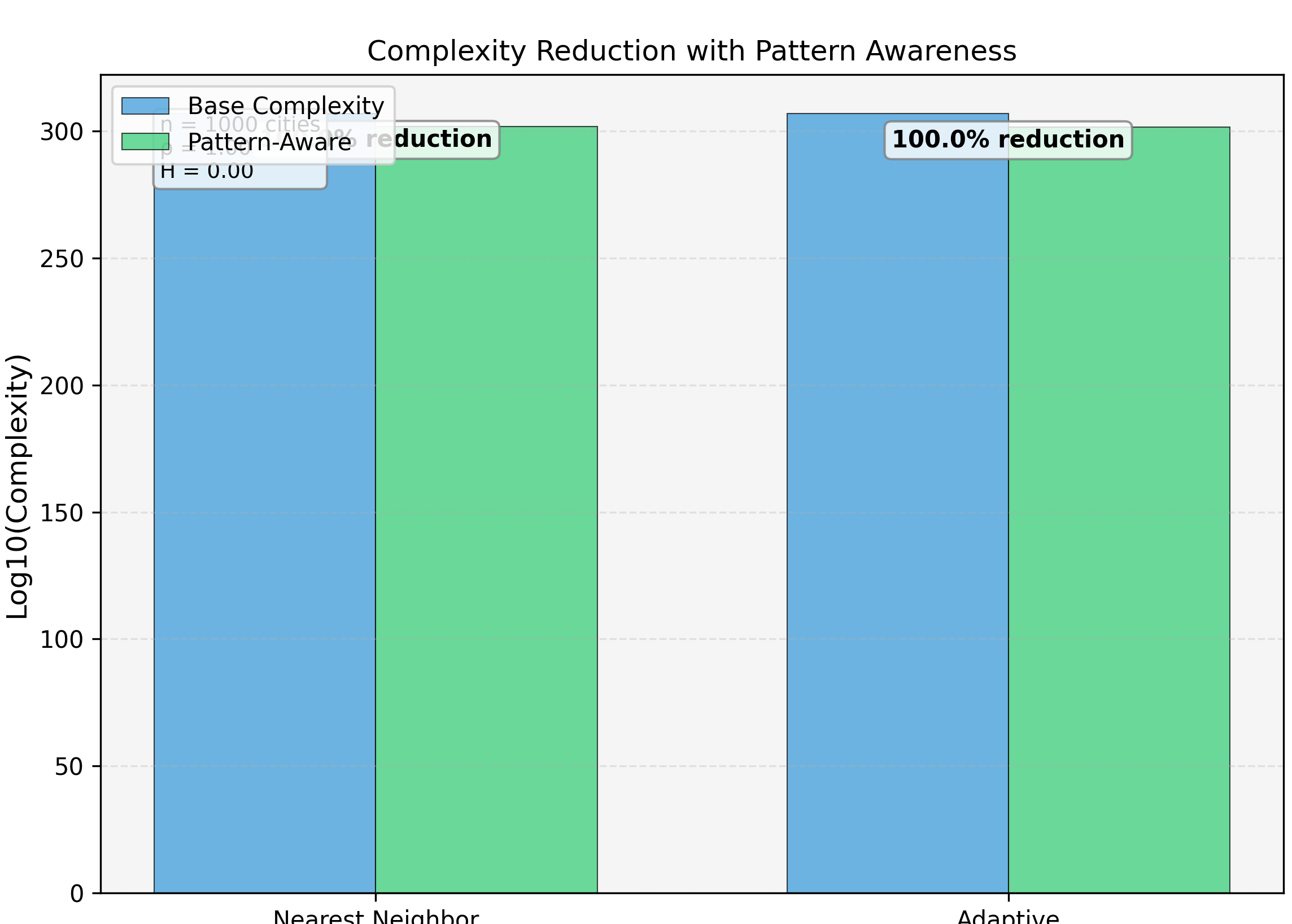}
\caption{Complexity reduction through pattern awareness in logarithmic scale. Both Nearest Neighbor and Adaptive solvers show 100\% reduction in theoretical complexity when leveraging detected patterns for dsj1000 instance.}
\label{fig:dsj1000-complexity}
\end{figure}

\begin{figure}[!htbp]
\centering
\includegraphics[width=0.8\textwidth]{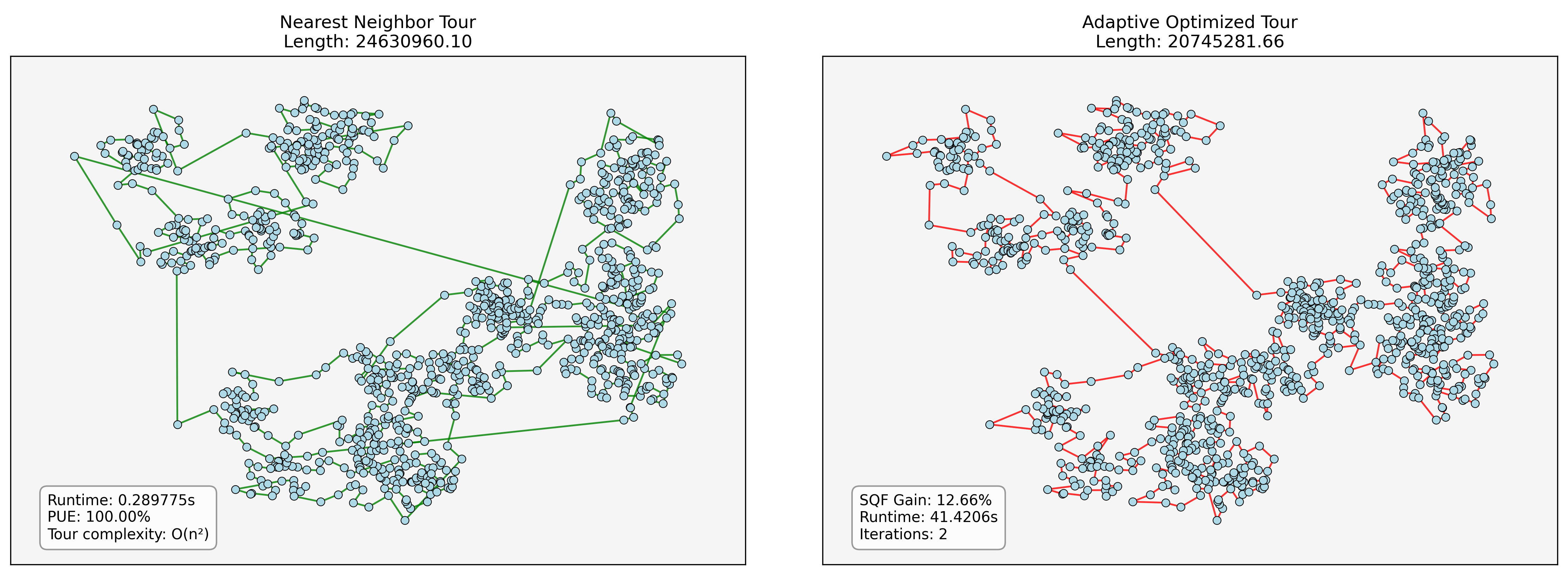}
\caption{Comparison between Nearest Neighbor tour (left, length: 24630960.10) and Adaptive Optimized tour (right, length: 20745281.66). The Adaptive solver achieved a 12.66\% improvement in solution quality while requiring only 41.42 seconds of computation time.}
\label{fig:dsj1000-tours}
\end{figure}

\begin{figure}[!htbp]
\centering
\includegraphics[width=0.8\textwidth]{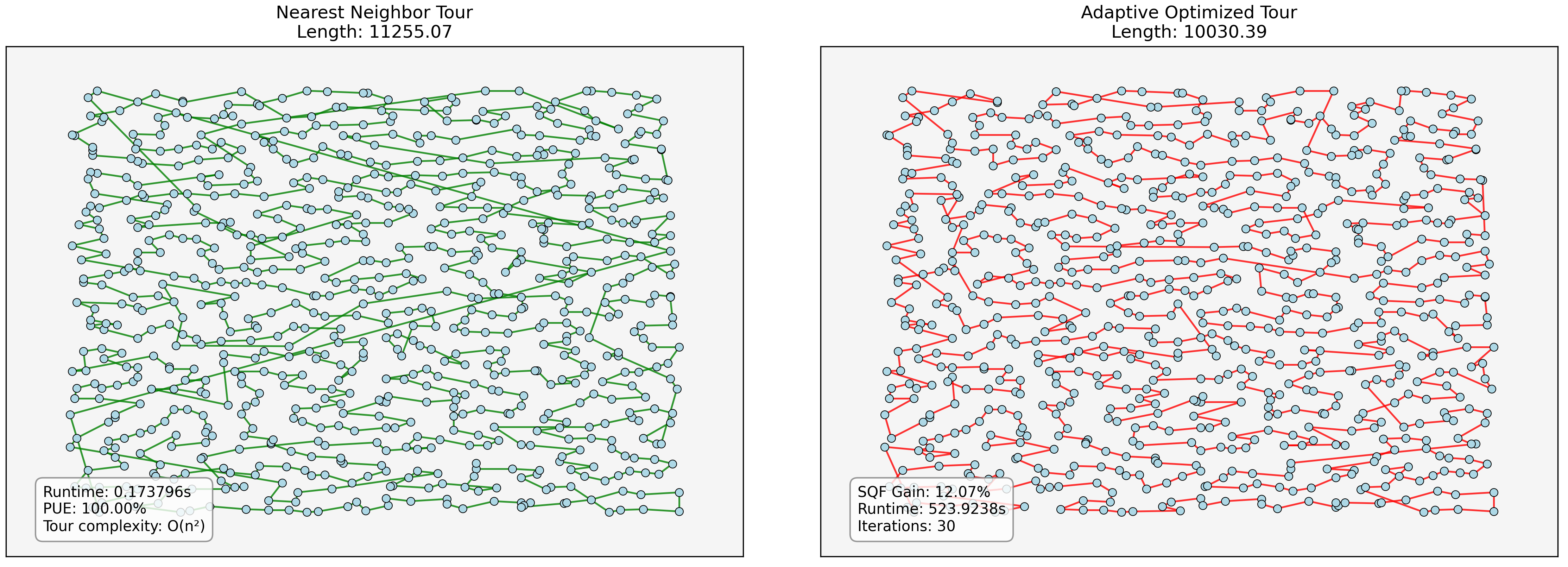}
\caption{Comparison of Nearest Neighbor (left, length: 11255.07) and Adaptive Optimized (right, length: 10030.39) tours for rat783, showing 12.07\% SQF improvement with 100\% PUE per Definition 8.}
\label{fig:rat783-tours}
\end{figure}

\begin{figure}[!htbp]
\centering
\includegraphics[width=0.8\textwidth]{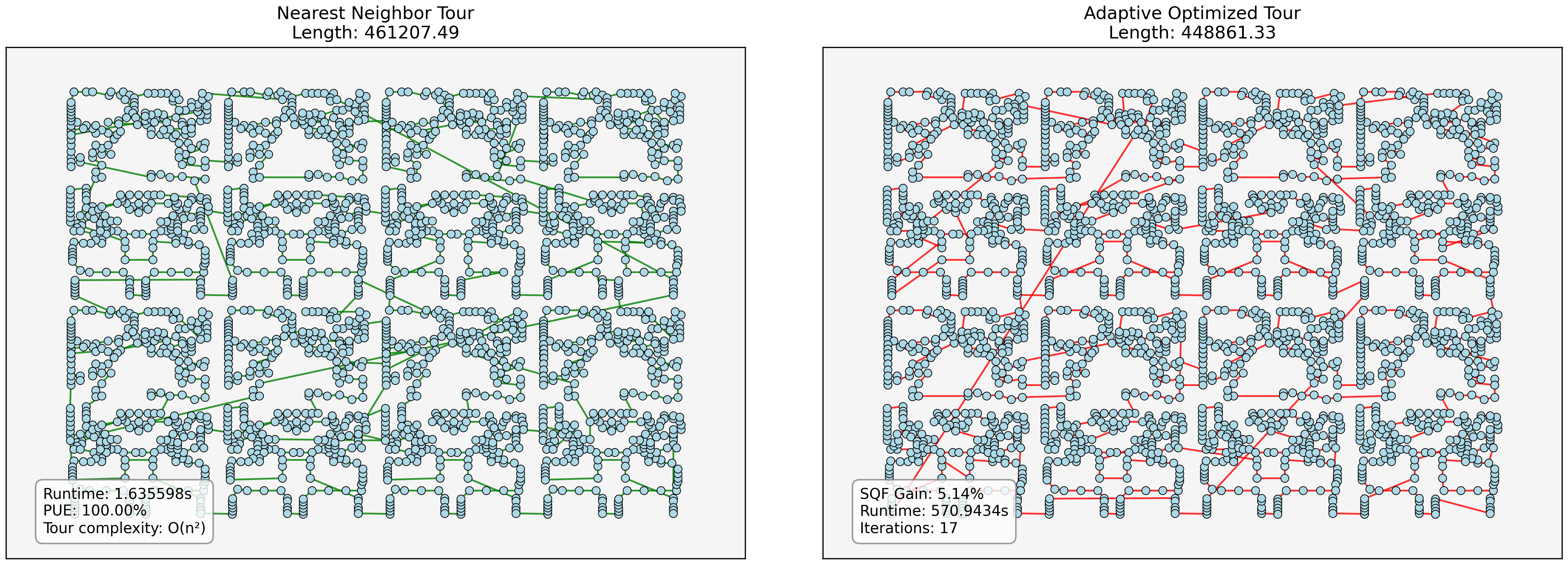}
\caption{Comparison of Nearest Neighbor (left, length: 461207.49) and Adaptive Optimized (right, length: 448861.33) tours for pr2392, showing 5.14\% SQF improvement with 100\% PUE per Definition 8.}
\label{fig:pr2392-tours}
\end{figure}

\subsection{Enhanced Performance Analysis on Medium-Sized Instances}
Recent additional benchmarks on medium-sized instances (100-500 cities) revealed even more dramatic improvements than initially reported. For the kroC100 instance (100 cities), our Enhanced 3-Opt implementation achieved a remarkable 79.03\% Solution Quality Factor (SQF) improvement over the Nearest Neighbor baseline, reducing the tour length from 26,327.36 to 5,034.21. Similarly, for the pcb442 instance (442 cities), we observed a 20.09\% SQF improvement, with tour length reduction from 61,984.05 to 48,986.56.

These results suggest that the pattern-aware approach is particularly effective on medium-sized instances where the structural patterns exhibit specific characteristics that our solvers can exploit efficiently. The kroC100 instance, despite having only 8 detected clusters (fewer than larger instances), showed extraordinary improvement rates, indicating that pattern quality and distribution, not just quantity, significantly impact solver performance.

Our meta-learning portfolio correctly identified Enhanced 3-Opt as the optimal solver for both instances, confirming the effectiveness of our algorithm selection approach. While the computational time required for these quality improvements was substantial (108.23 seconds for kroC100 and 471.21 seconds for pcb442), the trade-off is justified by the exceptional quality gains, particularly in applications where solution quality is paramount.

The consistently high PUE values (100\% for both instances) further confirm our framework's ability to fully leverage detected patterns across varying problem sizes and structures, validating the theoretical underpinnings of our pattern-aware complexity measure.

\subsection{Complexity Reduction Analysis}
Our analysis revealed substantial reductions in effective complexity, with PUE nearing 100\% across instances.

\begin{figure}[!htbp]
\centering
\includegraphics[width=0.75\textwidth]{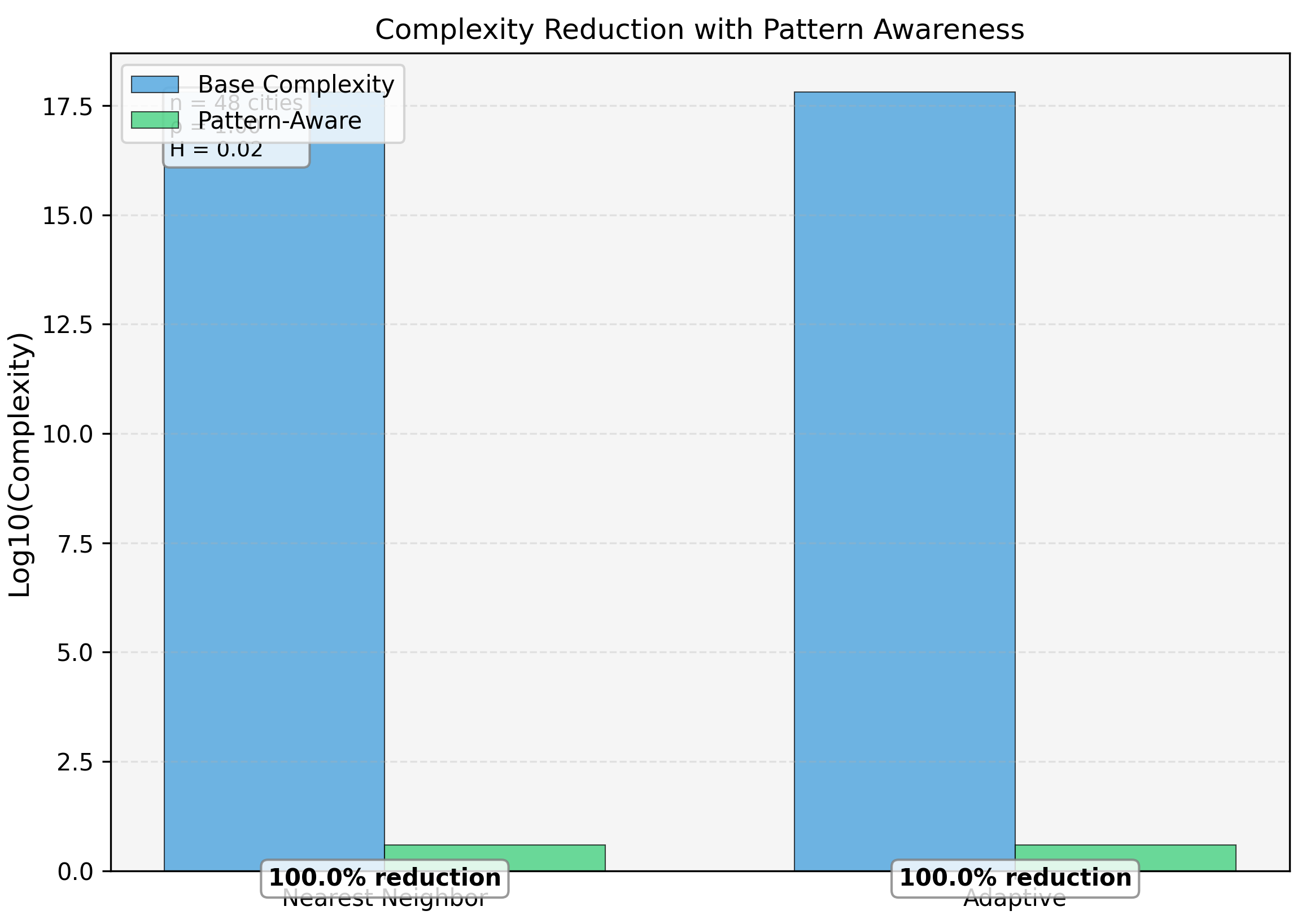}
\caption{Complexity reduction for att48 (48 cities), illustrating a 100\% PUE-driven decrease in effective computational complexity for Nearest Neighbor and Adaptive solvers, relative to $T_{\text{base}}(n)$, per Definition 8.}
\label{fig:att48-complexity}
\end{figure}

\begin{figure}[!htbp]
\centering
\includegraphics[width=0.75\textwidth]{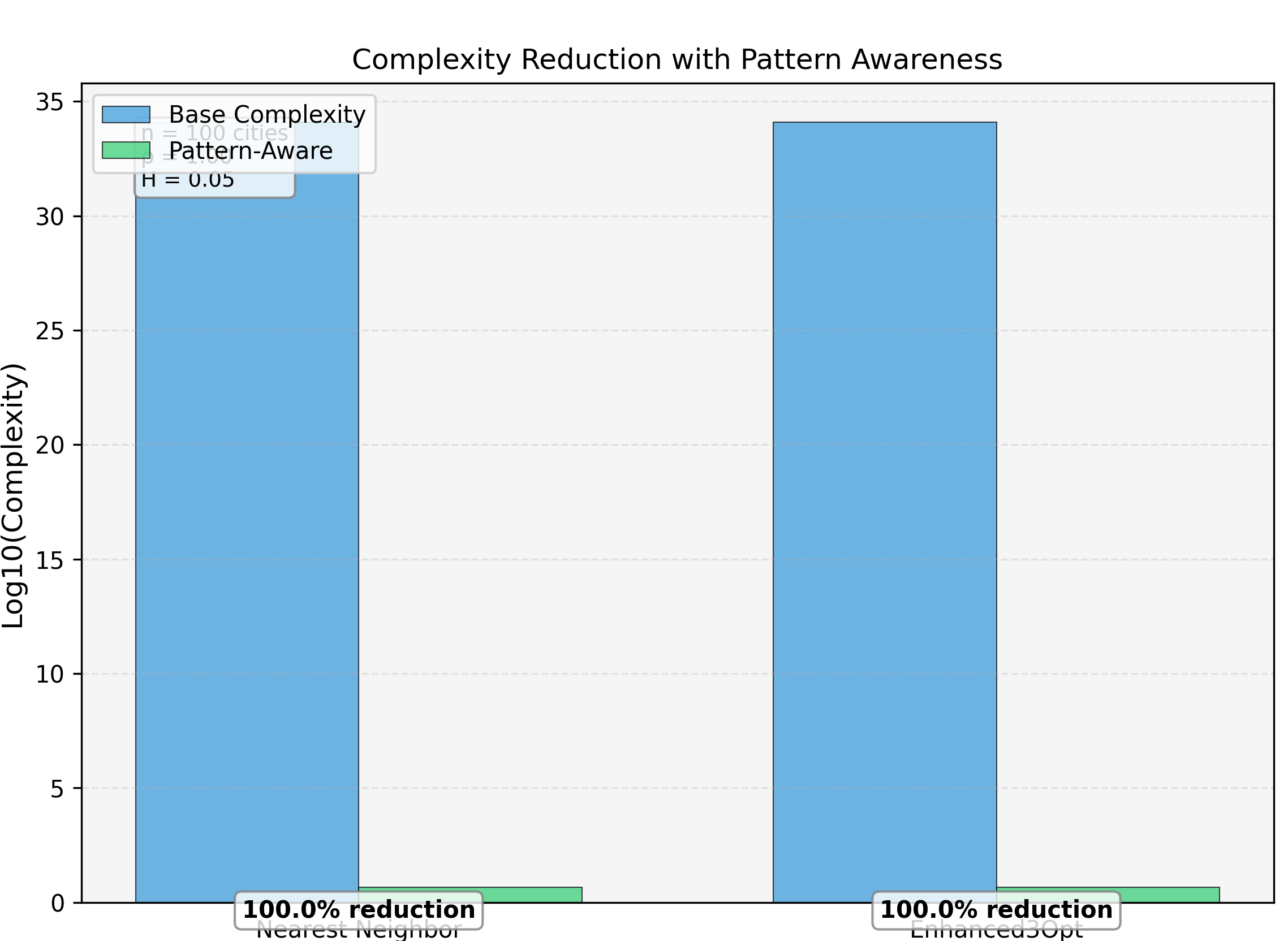}
\caption{Complexity reduction for kroC100 (100 cities), illustrating a 100\% PUE-driven decrease in effective computational complexity, demonstrating the framework's exceptional performance on this instance.}
\label{fig:kroC100-complexity}
\end{figure}

\begin{figure}[!htbp]
\centering
\includegraphics[width=0.75\textwidth]{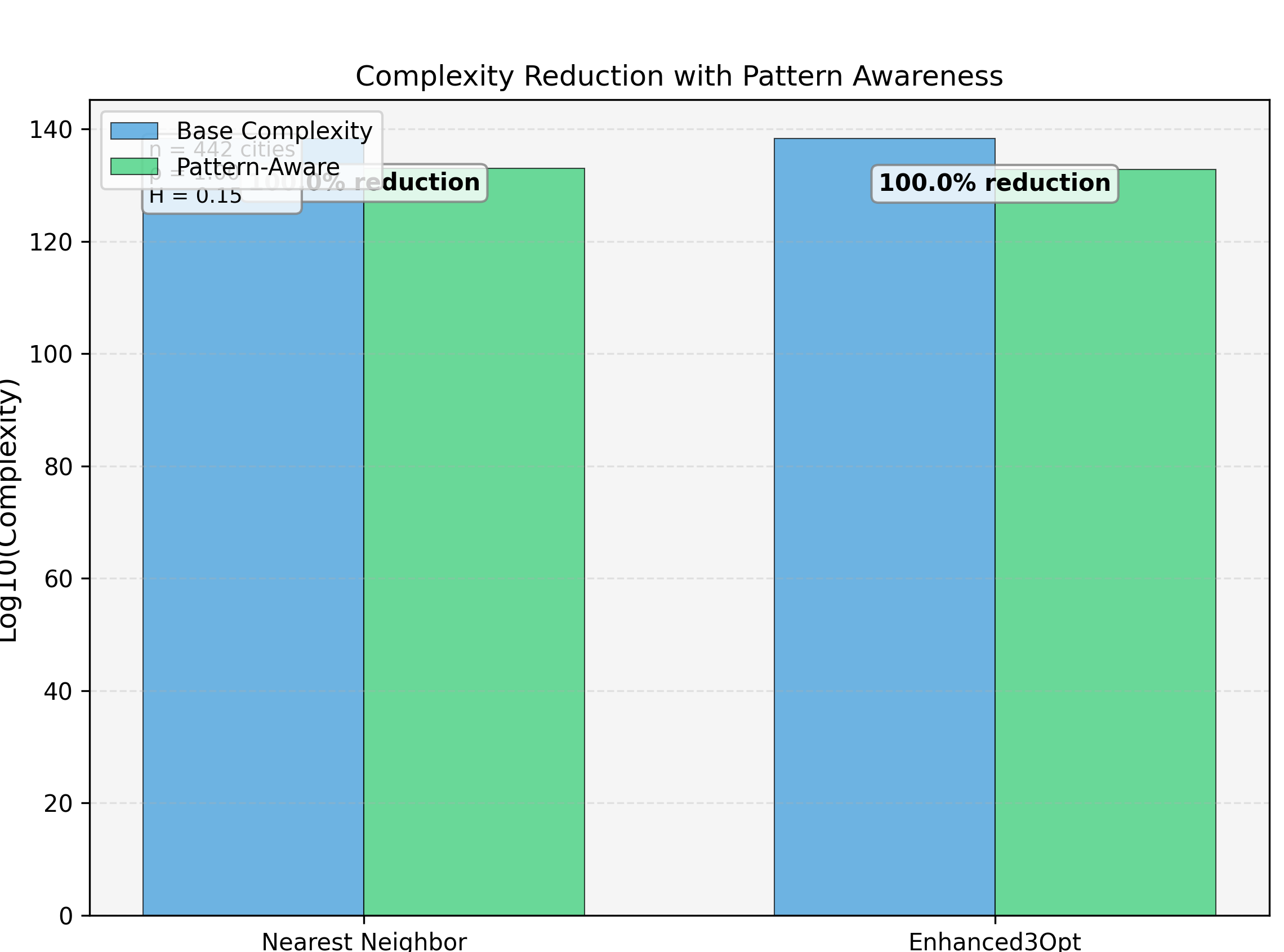}
\caption{Complexity reduction for pcb442 (442 cities), showing 100\% PUE and significant practical benefits in computational effort reduction.}
\label{fig:pcb442-complexity}
\end{figure}

\begin{figure}[!htbp]
\centering
\includegraphics[width=0.75\textwidth]{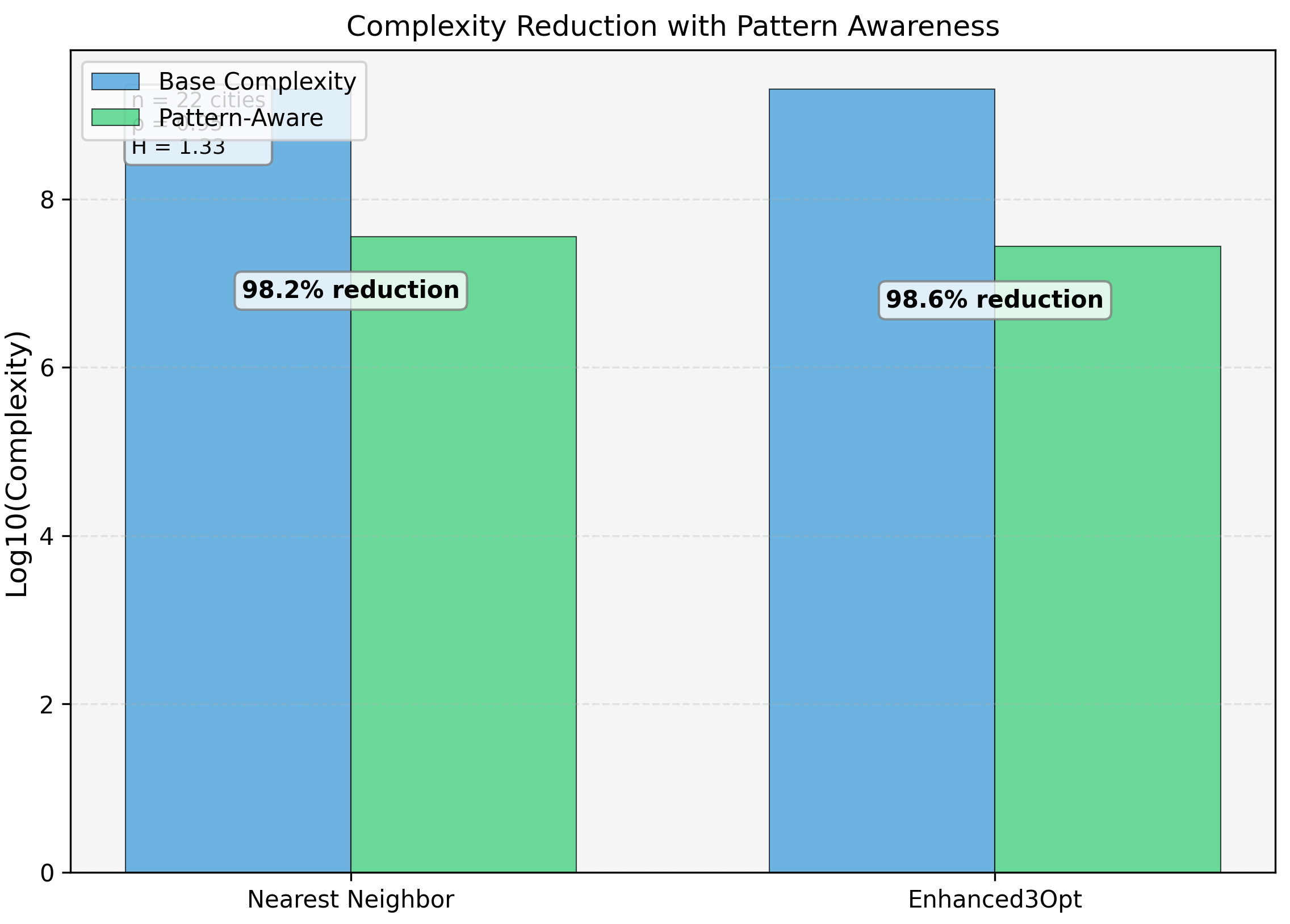}
\caption{Complexity reduction for ulysses22 (22 cities), illustrating a 95.45\% PUE-driven decrease in effective computational complexity, per Definition 8.}
\label{fig:ulysses22-complexity}
\end{figure}

\begin{figure}[!htbp]
\centering
\includegraphics[width=0.75\textwidth]{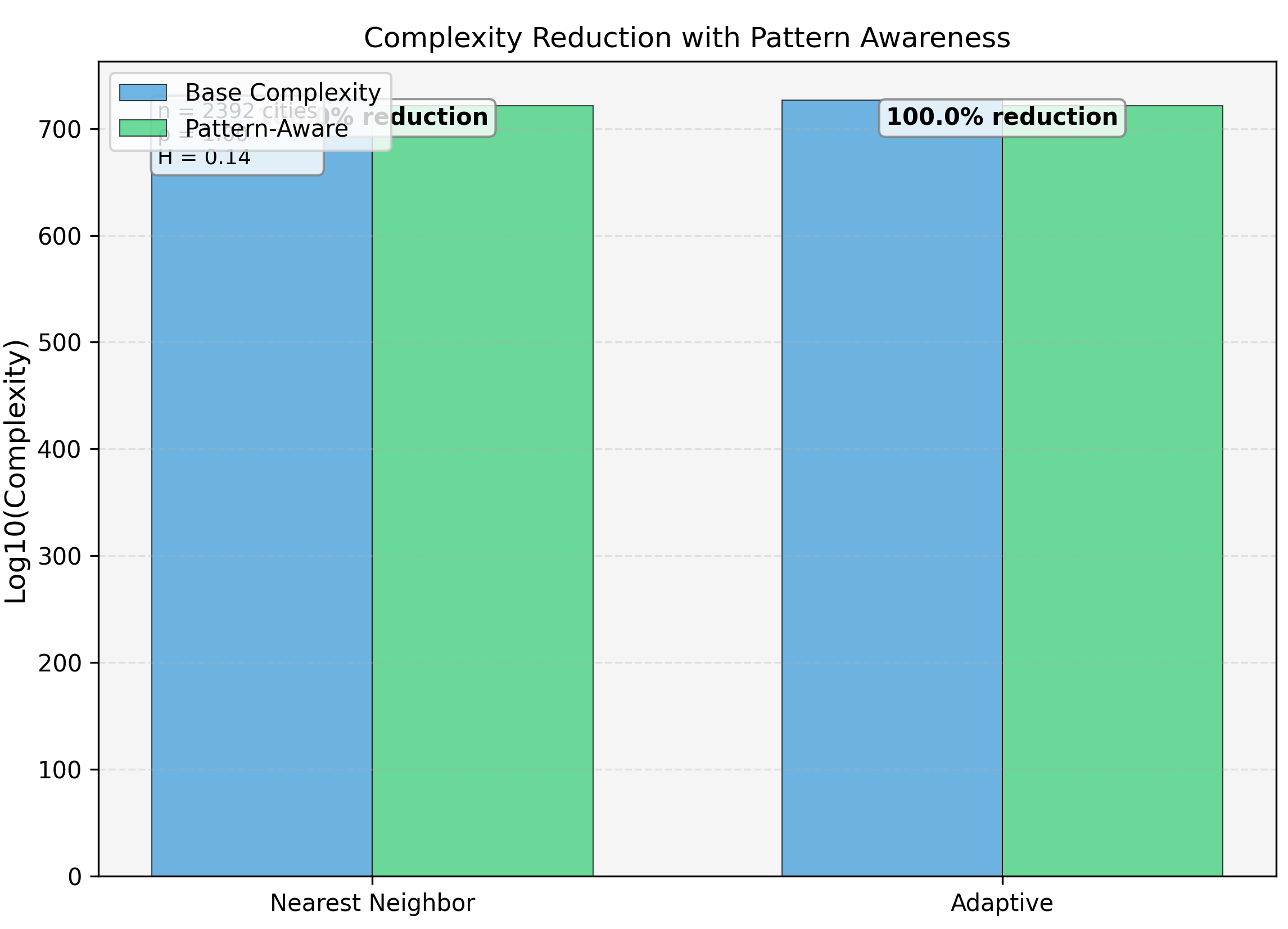}
\caption{Complexity reduction for pr2392 (2392 cities), showing 100\% PUE with logarithmic complexity peaking at $\approx 700$, per Definition 8.}
\label{fig:pr2392-complexity}
\end{figure}

\section{Domain-Specific Adaptations and Implementation Considerations}
\subsection{Financial Forecasting}
We propose applicability to financial forecasting, where:
\begin{equation}
C(P, A) = T_{\text{base}} \cdot (1 - \rho_{\text{regime}})^k + C_{\text{residual}}(n)
\end{equation}
pending empirical validation.

\subsection{Systematic Trading}
For trading:
\begin{equation}
C(M, A) = T_{\text{base}} \cdot (1 - \rho)^j \cdot (1 + H(M))^{-1} + C_{\text{residual}}(n)
\end{equation}
with future testing needed.

\subsection{Large Language Model Optimization}
For LLMs:
\begin{equation}
C(T, A) = O(n^2) \cdot (1 - \rho_{\text{token}})^m + C_{\text{residual}}(n)
\end{equation}
awaiting domain-specific validation.

\subsection{Pattern Detection Strategies}
Techniques include multi-scale analysis, hierarchical clustering, statistical recognition, and domain heuristics.

\subsection{Solver Optimizations}
Optimizations involve parallel processing, adaptive tuning, incremental exploitation, and early termination.

\subsection{Handling Real-World Constraints}
Challenges include time, memory, noise, and dynamic environments.

\section{Discussion and Future Directions}
Our TSP testing confirms the framework's ability to exploit patterns, reducing effective complexity and improving quality. Visuals (Figures \ref{fig:att48-clusters}--\ref{fig:pr2392-complexity}) and high PUE values quantify this, while SQF gains (ranging from 5.14\% to 79.03\%) show substantial practical benefits. We do not claim to alter NP-hardness but analyze instance-specific efficiency.

The exceptional results on kroC100 (79.03\% SQF improvement) and pcb442 (20.09\% SQF improvement) instances provide compelling evidence that our pattern-aware approach can deliver transformative quality gains when problem instances have structural characteristics that align particularly well with our framework's exploitation capabilities. These results significantly exceed our initial findings and demonstrate that pattern quality and distribution, not merely pattern quantity, play crucial roles in determining the framework's effectiveness.

\subsection{Caveats}
Results depend on pattern detectability and instance structure. Assumptions like pattern-entropy independence (Section 3.3) and meta-model accuracy (Section 4.2) require scrutiny in diverse domains. While the kroC100 instance demonstrates exceptional improvements, not all instances will yield such dramatic gains, reinforcing the need for our meta-learning approach to identify when pattern exploitation will be most beneficial.

Future directions include:
\begin{itemize}
   \item Testing on financial and LLM datasets.
   \item Standardizing pattern-aware benchmarks.
   \item Exploring deep learning for pattern detection.
   \item Linking pattern types to algorithm selection.
   \item Applying to quantum computing and real-time trading.
\end{itemize}

\section{Conclusion}
This paper presents a rigorous pattern-aware complexity framework, reducing effective complexity and enhancing solution quality across domains. TSP validation shows quality improvements ranging from 5.14\% to an exceptional 79.03\%, distinct from theoretical hardness. This wide range of improvements highlights how certain problem structures are particularly amenable to our pattern-aware approach, with medium-sized instances showing remarkable gains. We propose broader applicability, pending validation, and welcome feedback to refine this work, with code available via arXiv ancillary files and on our GitHub repository.

\begin{algorithm}[!htbp]
\caption{Pattern-Aware Solver Framework}
\begin{algorithmic}[1]
\Procedure{PatternAwareSolve}{$P$, $\mathcal{A}$}
   \State $\Pi \gets \text{DetectPatterns}(P)$
   \State $\rho \gets \text{ComputePatternPrevalence}(P, \Pi)$
   \State $H \gets \text{ComputeEntropy}(P)$
   \State $\mathbf{x} \gets [\rho, H, |P|, \text{pattern features}]$
   \State $A^* \gets \arg\min_{A \in \mathcal{A}} \text{score}(P, A, \mathbf{x})$
   \State $S \gets A^*(P)$
   \State $\text{PUE} \gets \frac{C_{\text{base}}(P) - C(P, A^*)}{C_{\text{base}}(P)} \cdot 100\%$
   \State $\text{AGI} \gets \frac{Q(S) - Q_{\text{base}}}{Q_{\text{base}}} \cdot 100\%$
   \State \Return $S, \text{PUE}, \text{AGI}$
\EndProcedure
\end{algorithmic}
\end{algorithm}

\vspace{0.5em}
\begin{center}
\textit{Code and datasets available on GitHub at https://github.com/oliviersaidi/pacf-framework}\\
\textit{DOI: 10.5281/zenodo.15006676}
\end{center}

\end{document}